%% file: hs_em_arxiv_6.22.17.tex
\documentclass[varwidth,11pt]{article}

\usepackage[slantedGreek]{mathpazo}
\usepackage[pdftex]{graphicx}
\usepackage{amsfonts}
\usepackage{setspace}
\usepackage{amssymb}
\usepackage{amsthm}
\usepackage{graphicx}
\usepackage{amsmath}
\usepackage{mathtools}
\usepackage{lscape}
\usepackage{nicefrac}
\usepackage{subcaption}
\usepackage{suffix}
\usepackage{color}
\usepackage{enumitem}
\usepackage{epsfig}

\definecolor{darkblue}{rgb}{0.0,0.0,0.55}
\RequirePackage[colorlinks,citecolor=darkblue,urlcolor=darkblue,linkcolor=darkblue]{hyperref} 

\usepackage[sort,comma]{natbib}

\allowdisplaybreaks

\include{math-commands}

\graphicspath{{../../figures/}{../figures/}{./figures/}{./}}
\renewcommand{\E}{\mathbb E}

\setcitestyle{citesep={;}}

\graphicspath{{../../figures/}{../figures/}{./figures/}{./}}

\title{\vspace{-1cm} \baselineskip=20pt  \bf Horseshoe Regularization for Feature Subset Selection}
\date{}
\begin{document}
\maketitle
\baselineskip=15pt
\begin{center}
\vspace{-1cm}
\author Anindya Bhadra  \footnotemark[1]   \footnotetext[1]{{\em Address:} Department of Statistics, Purdue University, 250 N. University St., West Lafayette, IN 47907, USA.} 
\and Jyotishka Datta  \footnotemark[2]   \footnotetext[2]{{\em Address:} Department of Mathematical Sciences, University of Arkansas, Fayetteville, AR 72701, USA.}
\and Nicholas G.~Polson  \footnotemark[3] \footnotetext[3]{{\em Address:} Booth School of Business, The University of Chicago, 5807 S. Woodlawn Ave., Chicago, IL 60637, USA.} and Brandon Willard  \footnotemark[3]
\end{center}
\begin{abstract} 
\noindent Feature subset selection arises in many high-dimensional applications of statistics, such as compressed sensing and genomics. The $\ell_0$ penalty is ideal for this task, the caveat being it requires the NP-hard combinatorial evaluation of all models. A recent area of considerable interest is to develop efficient algorithms to fit models with a non-convex $\ell_\gamma$ penalty for $\gamma\in (0,1)$, which results in sparser models than the convex $\ell_1$ or lasso penalty, but is harder to fit. We propose an alternative, termed the horseshoe regularization penalty for feature subset selection, and demonstrate its theoretical and computational advantages. The distinguishing feature from existing non-convex optimization approaches is a full probabilistic representation of the penalty as the negative of the logarithm of a suitable prior, which in turn enables efficient expectation-maximization and local linear approximation algorithms for optimization and MCMC for uncertainty quantification. In synthetic and real data, the resulting algorithms provide better statistical performance, and the computation requires a fraction of time of state-of-the-art non-convex solvers.
\end{abstract} 
\baselineskip=15pt
\section{Introduction}
Feature subset selection is typically performed by convex penalties such as the lasso \citep{tibshirani96}, the elastic net \citep{zou2005regularization}, or their variants. Convex penalties enjoy a number of advantages, such as uniqueness of solution, efficient computation and relatively straightforward theoretical analysis. Convex penalties, however, suffer from some undesirable features. For example, the lasso, which is based on a soft thresholding operation, leaves a constant bias that does not go to zero for large signals. A consequence is poor mean squared error in estimation.  The lasso also suffers from problems in presence of correlated variables. Non-convex penalties, on the other hand, can result in optimal theoretical performances for variable selection \citep{fan2001variable}. However, the computational burden of fitting non-convex penalties is more challenging. In this article, we take a Bayesian view of the optimization problem as finding the posterior mode under a given prior. Our approach is probabilistic, which enables a latent variable representation and results in efficient expectation-maximization \citep{dempster1977} and local linear approximation \citep{zou2008one} algorithms for optimization, as well as a Markov chain Monte Carlo (MCMC) scheme for posterior simulation. The performance comparison in simulations reveals the proposed regularization provides better statistical performance, while allowing much faster computation compared to state-of-the-art non-convex solvers. 

\subsection{Related Works in Non-Convex Regularization}
Consider the sparse normal means model where we observe $(y_i \mid \theta_i) \stackrel{ind}\sim \Nor(\theta_i, 1)$ for $i=1,\ldots, n$, where $\#(\theta_i \ne 0)\le p_n$ and $p_n = o(n)$ as $n \to \infty$. Non-convex regularization problems arise from a need to correctly identify the zero components in $\theta = (\theta_1, \ldots, \theta_n)$, given observations $y=(y_1,\ldots, y_n)$, also known as subset selection.  The  $\ell_0$ penalty, defined as $||\theta||_0 =\sum_{i=1}^{n} 1(|\theta_i|>0)$, is ideal for this task, and the more commonly used lasso or convex $\ell_1$ penalty, $||\theta||_1 = \sum_{i=1}^{n} |\theta_i|$, tends to select a denser model \citep{mazumder2012}. Unfortunately, na\"ively using the $\ell_0$ penalty requires a combinatorial evaluation of all $2^n$ models, which is NP-hard \citep{natarajan95}. Penalties of the form $\ell_\gamma$ for $\gamma \ge 1$ give rise to convex problems and efficient solvers are available. It remains a challenge to fit models with $\ell_\gamma$ penalties for $\gamma\in (0,1)$. While this does not necessarily a present combinatorial problem, the regularization problem is non-convex. Thus, the general purpose tools for convex optimization do not apply, nor is a unique solution guaranteed \citep[see, e.g.,][Chapter 1]{boyd2004convex}. Non-convex penalties include the smoothly clipped absolute deviation or SCAD \citep{fan2001variable} and the minimax concave penalty or MCP \citep{zhang2010nearly}. Recent computational advances in fitting models with non-convex penalties include \citet{breheny2011coordinate} and \citet{mazumder2012}. Both works use coordinate descent approaches to fit SCAD and MCP and provide conditions for convergence.  Alternatively, an overview of proximal algorithms for non-convex optimization is given by \citet{parikh2014proximal} and \citet{polson2015proximal}. Recent works have also demonstrated the equivalence between fitting a model with MCP penalty and evaluating the posterior mode in a Bayesian hierarchical model under a suitable prior \citep{strawderman2013hierarchical, schifano2010majorization}. Following along these lines, we show that evaluating the posterior mode under a suitable approximation to the horseshoe prior of \citet{carvalho2009handling, carvalho2010horseshoe} solves a non-convex optimization problem with desirable theoretical properties and derive fast computational algorithms. 

\section{The Horseshoe Prior and Penalty}\label{sec:pen}
Many penalized optimization problems in statistics take the form 
\begin{eqnarray}
\argmin_{\theta\in \mathbb{R}^n} \{l(\theta; y) + \pi(\theta)\},\label{eq:pen}
\end{eqnarray}
where $l(\theta; y) $ is a measure of fit of parameter $\theta$ to data $y$ (also known as the empirical risk), and $\pi(\theta)$ is a penalty function. Let $p(y\mid \theta) \propto \exp\{-l(\theta; y)\}$ and $p(\theta) \propto \exp\{-\pi(\theta)\}$, where $p$ denotes a generic density. If $l(\theta;y)$ is proportional to the negative of the log likelihood function under a suitable model, one arrives at a Bayesian interpretation to the optimization problem: finding the mode of the posterior density $p(\theta \mid y )$ under prior density $p(\theta)$ \citep{polson2015mixtures}. The properties of the penalty are then induced by those of the prior. The horseshoe prior \citep{carvalho2010horseshoe}
is defined as global-local Gaussian scale mixture under a half-Cauchy prior, with density
\begin{eqnarray}
p_{HS}(\theta_i \mid \tau) = \int_0^\infty \frac{1}{u_i\tau} \phi \left ( \frac{\theta_i }{u_i\tau} \right ) \frac{2}{\pi (1+u_i^2)} d u_i, \label{eq:horsedens}
\end{eqnarray}
where $\tau >0$ and $\phi(\cdot)$ denotes the standard normal density. Equivalently,
$$
 \theta_i \mid u_i , \tau  \stackrel{ind}\sim \Nor ( 0 , u_i^2 \tau^2 ), \quad \; u_i \mid \tau \stackrel{ind}\sim C^+ (0, 1), \quad \tau>0,
$$
where $C^{+}$ denotes a half-Cauchy random variable. The $ u_i $ are local shrinkage parameters which shrink irrelevant signals to zero while keeping the magnitude of true signals. The parameter $ \tau$ is a global shrinkage parameter which gauges the level of sparsity. Several optimality results are available when the posterior mean under the horseshoe prior is used as an estimator, such as minimax optimality in estimation under $\ell_2$ loss \citep{van2014horseshoe} and asymptotic optimality in testing under 0--1 loss \citep{datta2013asymptotic}. However, little is known of the properties of the posterior mode under the horseshoe prior, which amounts to a solution of (\ref{eq:pen}) under the horseshoe penalty using a squared error empirical risk, given by $\sum_{i=1}^{n} (\theta_i - y_i)^2$.

\begin{figure*}[!t]
   	\centering
	\includegraphics[width=\textwidth, height=5cm]{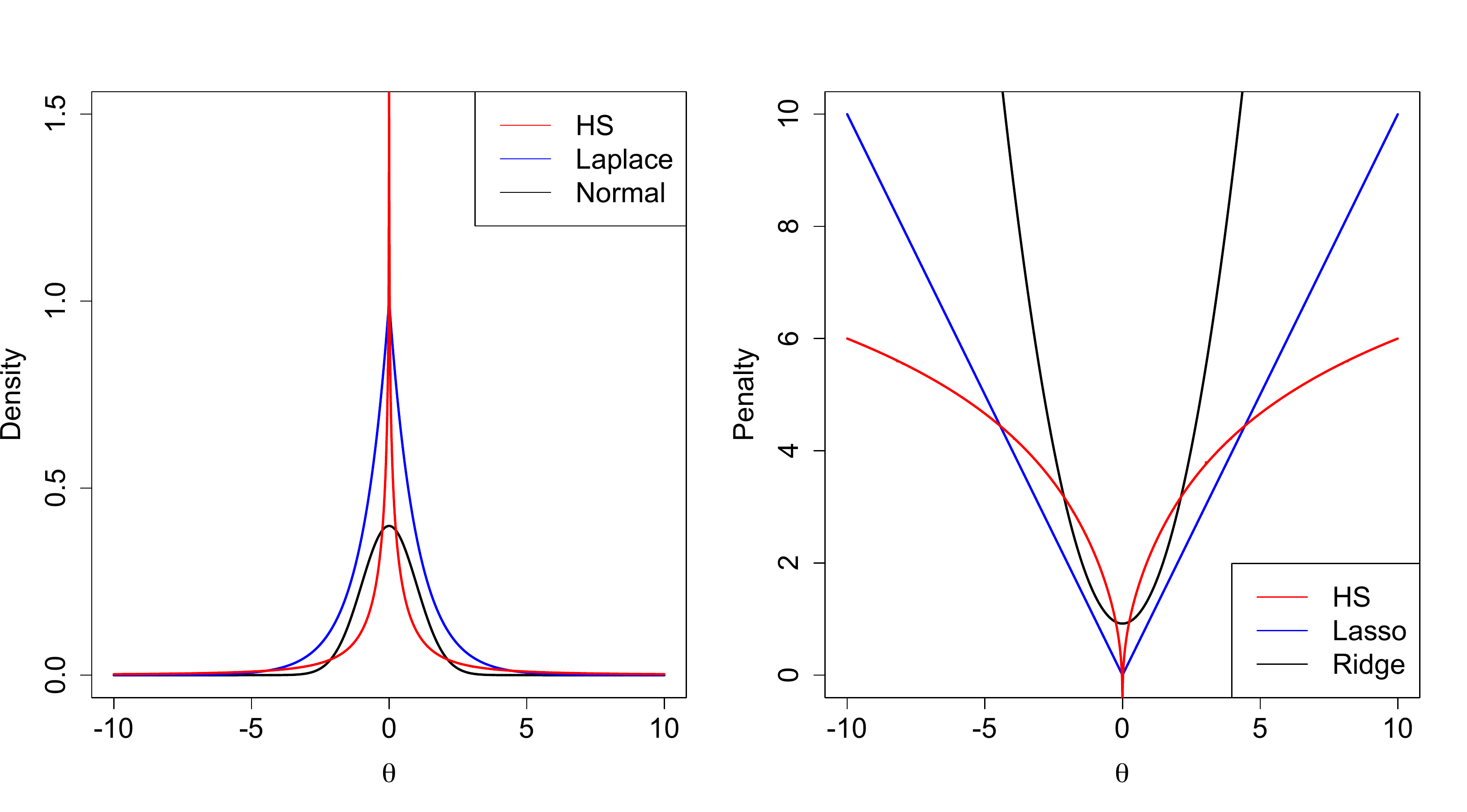}
   	\caption{Some densities and penalties given by the negative of their logarithms. Left panel: the horseshoe (HS) with $\tau=1$, the standard Laplace and standard normal densities. Right panel: the corresponding non-convex horseshoe and the convex lasso and ridge penalties.}
   	\label{fig:penalties}
   \end{figure*}
   
\citet{carvalho2010horseshoe} show that the horseshoe prior density admits tight upper and lower bounds
\begin{eqnarray}
\frac{\log \left ( 1 + \frac{4\tau^2}{\theta_i^2} \right )}{\tau {(2 \pi)^{3/2}}}< p_{ HS} (\theta_i \mid \tau) <\frac{2 \log \left ( 1 + \frac{2\tau^2}{\theta_i^2} \right )}{\tau {(2 \pi)^{3/2}}}, \label{eq:bound}
\end{eqnarray}
for $\theta_i  \in \mathbb{R},\; \tau>0$ and prove $\lim_{|\theta_i|\to0} p(\theta_i) = \infty$ and $\lim_{|\theta_i|\to \infty} p(\theta_i) \sim (\theta_i)^{-2}$ for any fixed $\tau$. The corresponding penalty is $\pi_{HS}(\theta\mid \tau) = \sum_{i=1}^{n} \pi_{HS}(\theta_i \mid \tau)$, where,
\begin{eqnarray}
\pi_{HS}(\theta_i \mid \tau) = -\log p_{HS}(\theta_i\mid \tau) = - \log \log \left ( 1 + \frac{2\tau^2}{\theta_i^2} \right ),\label{eq:horsepen}
\end{eqnarray}
up to terms involving $\theta_i$. Since the horseshoe prior density has tails decaying as $\theta_i^{-2}$, the corresponding penalty behaves as the logarithmic penalty for large values of $|\theta_i|$ and can be seen to be non-convex. Figure~\ref{fig:penalties} right panel shows the non-convex horseshoe penalty, in combination with convex lasso and ridge penalties. They are respectively obtained by taking the negative of the logarithm of the horseshoe, the Laplace and normal densities, shown on the left panel. It can be seen that the horseshoe penalty is more aggressive near zero compared to the convex penalties, encouraging sparsity. In fact, both the density and penalty are unbounded for the horseshoe at zero, suggesting a global solution to the optimization problem in (\ref{eq:pen}) that is identically equal to zero. However, this does not preclude the possibility of other local solutions, and in fact encourages one to look for local optimization algorithms that might lead to solutions that are more interesting than all zeros. For values far away from zero, the horseshoe penalizes lightly, a fact that can also be attributed to the heavy tails of the local $u_i$ terms. This suggests the horseshoe penalty bridges the gap between $\ell_0$ and $\ell_1$ penalties, a property also shared by other non-convex penalties such as SCAD or MCP.

\subsection{Properties of the Horseshoe Penalty}

\citet{fan2001variable} describe the desirable properties for a penalty function and list conditions for checking whether those properties hold. These are as follows:
\begin{enumerate}
\item \emph{(Near) Unbiasedness.} The resultant estimator is (nearly) unbiased when the true parameter is large. A sufficient condition is that the penalty satisfies $\pi'(|\theta|) =0$ when $|\theta|\to \infty$ where $\pi'$ is the first derivative of the penalty $\pi$.
\item \emph{Sparsity.} The resultant estimator is sparse. A sufficient condition is that $\inf_{\theta}\{|\theta| + \pi'(|\theta|)\}>0$.
\item \emph{Continuity.} The resultant estimator is continuous in the data $y$ to encourage stability in prediction. A  necessary and sufficient condition is that $\argmin_{\theta}\{|\theta| + \pi'(|\theta|)\}=0$.
\end{enumerate}
Property (3) is violated by hard thresholding rules, whereas Property (1) is violated by the lasso and associated soft thresholding rules and also by all penalties of the form $\ell_\gamma$ for $\gamma>1$. Penalties that satisfy Properties (1)--(3) include MCP and SCAD, however the computational algorithms used to fit models employing these penalties are quite challenging and can suffer from numerical issues. We first show that the horseshoe penalty enjoys Properties (1)--(2) above, arguing for its theoretical advantage; before preceding to develop efficient computational algorithms.
\begin{proposition}
The horseshoe posterior mode, defined as $\argmin_{\theta} \{(\theta - y)^2/2 + \pi_{HS}(\theta)\}$, where $\pi_{HS} (\theta)$ denotes the horseshoe penalty, satisfies Properties (1)--(2) above, but not Property (3).
\end{proposition}
\begin{proof}
It is enough to check the properties for a single coordinate $\theta_i$. First, note from (\ref{eq:horsepen}) that
$$
\pi'_{HS}(|\theta_i|) = \frac{{4\tau^2/|\theta_i|^3}}{\left(1+{2\tau^2/\theta_i^2}\right)\log \left(1+{2\tau^2}/{\theta_i^2}\right) },
$$
and Property (1) follows. Next, for Property (2),
\begin{eqnarray*}
|\theta_i| + \pi'_{HS}(|\theta_i|) &=& |\theta_i| +\frac{{4\tau^2/|\theta_i|^3}}{\left(1+{2\tau^2/\theta_i^2}\right)\log \left(1+{2\tau^2}/{\theta_i^2}\right) }.
\end{eqnarray*}
Thus, $|\theta_i| + \pi'_{HS}(|\theta_i|)\to\infty$ as $|\theta_i| \to 0, \infty$ for any given $\tau>0$. For $|\theta_i| \ne 0, \infty$, the denominator of the second term is strictly positive. Thus, we need to check
$$
 \theta_i^4 {\left(1+{2\tau^2}/{\theta_i^2}\right)\log \left(1+{2\tau^2}/{\theta_i^2}\right) } + 4\tau^2>0.
$$
Result 4.1.33 of \citet{abram65} gives
$$
x < (1+x) \log(1+x), \quad x>-1, x\ne 0.
$$
Using $x=2\tau^2/\theta_i^2$ yields
\begin{eqnarray*}
 \theta_i^4 {\left(1+{2\tau^2}/{\theta_i^2}\right)\log \left(1+{2\tau^2}/{\theta_i^2}\right) } + 4\tau^2 > 2\tau^2 \theta_i^2 + 4\tau^2,
\end{eqnarray*}
which is strictly positive for any $\tau>0$, proving Property (2) holds. Since $\lim_{|\theta_i|\to 0} \{|\theta_i| + \pi'_{HS}(|\theta_i|)\} = \infty$, Property (3) fails to hold.
\end{proof}
An implication of this result is that the resultant estimator is sparse and is nearly unbiased in estimating large signals. However, the lack of continuity means the estimator suffers from some of the same issues as hard thresholding. We verify the hard thresholding-like behavior of the estimator via simulations and show that if the posterior mean is used as an estimator rather than the posterior mode, then it solves the continuity problem and usually results in an estimator with better squared error loss. However, the posterior mean does not result in a sparse solution and hence, is not suitable for subset selection. 

\section{The Horseshoe-Like Prior and Its Scale Mixture Representation}\label{sec:sm}
There is no closed form for the horseshoe density and numerical integration over $u_i$ in (\ref{eq:horsedens}) is required to evaluate the density at any given $\theta_i$. The tight upper and lower bounds in (\ref{eq:bound}) are also not densities. However, a proper prior density that mimics the behavior of the horseshoe density with a pole at the origin and polynomial tails is given by
\begin{eqnarray}
p_{\widetilde{HS}} (\theta_i \mid a) = \frac{1}{2 \pi{a^{1/2}}}\log \left ( 1 + \frac{a}{\theta_i^2} \right ), \label{eq:ptilde}
\end{eqnarray}
for $\theta_i  \in \mathbb{R},\; a>0$. We call this the horseshoe-like prior. Setting $a=2\tau^2$ and $a=4\tau^2$ in (\ref{eq:ptilde}) one recovers the bounds in (\ref{eq:bound}) that differ only by a constant factor. Since the bounds in (\ref{eq:bound}) are tight in $\theta_i$, and a constant multiplicative factor of the density (or, equivalently, a constant additive term to the penalty) has no bearings on the solutions to the optimization problem, one can use $\pi_{\widetilde{HS}}(\theta_i) = -\log( p_{\widetilde{HS}}(\theta_i))$ as a useful surrogate of the horseshoe penalty. The chief advantage of using a proper density is that it enables one to use the technique of latent variables to solve the optimization problem, such as the EM algorithm or the techniques based on data augmentation \citep{tanner1987calculation}, provided one can find a suitable probabilistic representation. 

The methodology developed in the remainder of this article relies on the following key result. For a real valued function $f(\cdot)$, the Frullani integral identity \citep[][pp.~406--407]{Jeffreys:MMP} gives
$$
\int_{0}^{\infty} \frac{f(cx) - f(dx)}{x} dx = \{f(0) - f(\infty)\}\log(d/c),
$$
for $c>0, d>0$. Using $f(x) = \exp(-x)$ yields a latent variable representation for the global-local scale mixture for $p_{\widetilde{HS}} (\theta_i \mid a)$ in (\ref{eq:ptilde}) as:
\begin{eqnarray*}
\frac{1}{2\pi a^{1/2}} \log\left(1 + \frac{a}{\theta_i ^2}\right) &=& \int_0^\infty \exp\left({-\frac{u_i\theta_i^2}{a}}\right) \frac{(1-e^{-u_i})}{2\pi {a^{1/2}} u_i} du_i\\
&=& \int_0^\infty \left({\frac{u_i}{a\pi}}\right)^{1/2} \exp\left({-\frac{u_i\theta_i^2}{a}}\right) \frac{(1-e^{-u_i})}{2{\pi^{1/2}} u_i^{3/2}} du_i, \quad a>0.
\end{eqnarray*}
or equivalently,
\begin{equation}
(\theta_i \mid u_i, a) \stackrel{ind}\sim \Nor \left(0, \frac{a}{2u_i}\right),\; p(u_i)  =  \frac{1-e^{-u_i}}{2\pi^{1/2}u_i^{3/2}} \label{eq:u_i},
\end{equation}
for $ 0< u_i < \infty, \; a>0$. Once again, the $u_i$ terms act as local scale parameters and the global term $a$ controls the overall level of sparsity. A useful outcome of this probabilistic representation is that the $u_i$ terms can be viewed as latent variables and thus suggests the possibility of EM and MCMC schemes, provided the appropriate quantities in the posterior can be easily computed. 

\subsection{Alternative Scale Mixtures and the Marginal Density Under the Horseshoe-Like Prior}

The horseshoe-like prior density on $\theta_i$, given by \eqref{eq:ptilde}, can be represented as a scale mixture of both Cauchy and Laplace densities on $\theta_i$, as the following two lemmas show.

\begin{lemma}\label{lemma:cauchy}
For a fixed $\tau$, the horseshoe-like prior can be written as a uniform scale mixture of a Cauchy prior on $\theta_i$, i.e. $\theta_i \mid \lambda_i, \tau \sim \CauchyRV(0,\lambda_i \tau)$ and $\lambda_i \sim \UnifRV(0,1)$. 
\beq
p_{\widetilde{HS}} (\theta_i \mid \tau^2) = \frac{1}{2 \pi \tau}\log \left ( 1 + \frac{\tau^2}{\theta_i^2} \right ) = \frac{1}{\pi}  \int_{0}^{1} \frac{\lambda_i \tau}{\lambda_i^2 \tau^2 + \theta_i^2} d\lambda_i. \label{eq:cauchy}
\eeq
\end{lemma}
The proof is elementary and therefore omitted. The Cauchy scale mixture representation provides a natural adaptive sparsity model for the horseshoe-like prior. The horseshoe-like prior can be also expressed as a mixture of Laplace densities on $\theta_i$ due to a result by \cite{steutel2003infinite}. 

\begin{lemma}\label{lemma:laplace}
The horseshoe-like prior density in \eqref{eq:ptilde} can be written as a scale mixture of a double exponential or Laplace prior on $\theta_i$, as given below: 
\begin{align}
p_{\widetilde{HS}} (\theta_i \mid \tau^2)= \frac{1}{2 \pi \tau} \log \left(1 + \frac{\tau^2}{\theta_i^2} \right) & = \frac{1}{2\tau} \intpos \lambda_i \exp\{ -\lambda_i \abs{\theta_i}/\tau \} h(\lambda_i) d\lambda_i \nonumber \\
\text{where } \; h(\lambda) & = \frac{2}{\pi} \left( \frac{1 - \cos(\lambda)}{\lambda^2} \right) = \frac{1}{2\pi} \left( \frac{\sin(\lambda/2)}{\lambda/2} \right)^2, \; 0 \le \lambda < \infty. \label{eq:fejer}
\end{align}
Here the mixing density on $\lambda_i$ is a special type of density arising from Polya characteristic functions, called the Fejer-de la Vallee Poussin (or FVP) density \citep[Theorem 6.9]{devroye1986}.
\end{lemma}
A useful outcome of Lemma \ref{lemma:cauchy} is the following result for the marginal density on $y_i$ under an $\text{Inverse-Gamma}(1/2, 1/2)$ prior on $\sigma^2$. 
We can use a Cauchy convolution result \citep{bhadra2016global} to prove the following: 
\begin{proposition}\label{prop:marginal}
Let the observations $(y_i \mid \theta_i, \sigma^2) \sim \NormRV(\theta_i, \sigma^2)$ and $\sigma^2 \sim \text{Inverse-Gamma}(1/2, 1/2)$, where the $\theta_i$'s are given the horseshoe-like prior in \eqref{eq:ptilde}, i.e. $p(\theta_i \mid \tau) =  \frac{1}{2 \pi \tau}\log \left( 1 + \frac{\tau^2}{\theta_i^2} \right)$. Then the marginal density of $y_i$ is given by: 
\begin{align}
m(y_i \mid \tau) & = \frac{1}{2\pi\tau} \log \left(1 + \frac{\tau^2}{1+y_i^2} \right). \label{eq:marginal}
\end{align}
\end{proposition}
A proof is given in Appendix~\ref{app:marginal}. A consequence is that the marginal density $m(y_i\mid \tau)$ behaves similar to the prior density $p(\theta_i \mid \tau)$ for large values of $|y_i|$ and thus also displays heavy tails.

Implications of the Laplace scale mixture in Lemma~\ref{lemma:laplace} are discussed in Section~\ref{sec:lla}, where it is used to derive a local linear approximation (LLA) algorithm. 


\section{Computational Algorithm I: Algorithms for Feature Selection}\label{sec:em}
We derive fast computational algorithms for evaluating the maximum a-posteriori (MAP) estimate under the horseshoe-like prior. 
According to Sections~\ref{sec:pen} and~\ref{sec:sm}, the solution to the this problem is identical to that of the optimization problem in (\ref{eq:pen}) under the horseshoe penalty $\pi_{HS}(\theta)$, where the empirical risk measure is the squared error loss. The proposed technique uses the latent variable representation in (\ref{eq:u_i}) to derive an EM algorithm for MAP estimation. 

\subsection{EM for Subset Selection in Normal Means Model}\label{sec:seq}
First, consider the model: $(y_i \mid \theta_i) \stackrel{ind} \sim \Nor (0,1)$. From (\ref{eq:u_i}), the hierarchical model for $i=1,\ldots,n$ is
\begin{eqnarray*}
(y_i \mid \theta_i) \stackrel{ind}\sim \Nor (\theta_i, 1), \; (\theta_i \mid u_i, a) \stackrel{ind}\sim \Nor \left(0, \frac{a}{2u_i}\right),\; p(u_i)  =  \frac{1-e^{-u_i}}{2{\pi^{1/2}}u_i^{3/2}},  \; 0< u_i < \infty, \quad a>0.
\end{eqnarray*}
The complete data posterior is
\begin{eqnarray*}
p(\theta_i, u_i \mid y_i, a) &\propto& \exp\left\{{-\frac{(y_i-\theta_i)^2}{2}}\right\} \exp\left({-\frac{u_i\theta_i^2}{a}}\right) \frac{(1-e^{-u_i})}{ u_i}.
\end{eqnarray*}
If one views the $u_i$ terms as latent variables, the E-step consists of computing their posterior expectations. It is given by $\tilde u_i = E(u_i \mid \theta_i, y_i, a)$, where,
\begin{eqnarray*}
\tilde u_i = \frac{1}{2\pi{a^{1/2}}}  \int_{0}^{\infty} u_i \exp\left({-\frac{u_i\theta_i^2}{a}}\right) \frac{(1-e^{-u_i})}{u_i} du_i =\frac{1}{2\pi{a^{1/2}}} \left(\frac{a}{\theta_i^2} - \frac{a}{\theta_i^2 +a}\right).
\end{eqnarray*}
The M-step maximizes the complete data posterior jointly in $(\theta, a)$ with the $u_i$ terms replaced by $\tilde u_i$. While the joint maximization does not have a closed-form solution, the conditional maximizations $(\theta \mid a)$ and $(a\mid \theta)$ are simple. The optimal $\theta_i$ for a given $a$ is simply the Gaussian posterior mode, 
$$
\hat \theta_i \mid a= \left(1 + \frac{2\tilde u_i}{a}\right)^{-1} y_i.
$$
Maximization of $a$ with given $\theta$ is easy due to the fact that 
$$
\theta_i \sqrt{2u_i} \mid a \sim \Nor (0,a).
$$
Thus,
\begin{eqnarray*}
\hat a \mid \theta = \frac{1}{n} \sum_{i=1}^{n} 2 \tilde u_i \theta_i^2 =\frac{a^{3/2}}{n\pi} \sum_{i=1}^{n} \frac{1}{\theta_i^2 +a}.
\end{eqnarray*}
Thus, the $(t+1)^{\mathrm{th}}$ expectation-maximization recursion for $t\ge 0$ is given by as a coordinate descent, or as expectation-conditional maximization \citep{meng1993maximum}, as
\begin{align*}
\hat a^{(t+1)} \mid \hat\theta_1^{(t)}, \ldots, \hat\theta_n^{(t)} &= \frac{\{\hat a^{(t)}\}^{3/2}}{n\pi} \sum_{i=1}^{n} \left( \frac{1}{\{\hat\theta_i^{(t)}\}^2 +\hat a^{(t)}}\right), \\
\hat \theta_i^{(t+1)} \mid \hat a^{(t+1)} &= y_i \left( 1 +  \frac{\{\hat a^{(t+1)}\}^{1/2}} {\pi \{\hat\theta_i^{(t)}\}^2 \left[\{\hat\theta_i^{(t)}\}^2 + \hat a^{(t+1)}\right]}\right)^{-1},
\end{align*}
for $i=1, \ldots, n$, which is repeated until convergence and $\hat\theta^{(0)}$ and $\hat a^{(0)}$ are suitable initial values. 
Since the penalty is unbounded at zero, the global solution to the optimization problem is given by $\hat\theta_i = 0$ for all $i$. However, since the EM is a local, deterministic algorithm, it converges once a local mode is identified. In fact, the existence of a global mode identically equal to zero provides arguments against using a global optimization algorithm, such as simulated annealing \citep{kirkpatrick1983optimization}. The convergence of the EM algorithm of course depends on the choice of starting values. However, the fact that there is no unique solution is a result of the non-convex penalty itself, rather than an artifact caused by a failure of the optimization algorithm. Local solutions can be compared by evaluating the likelihoods at the solutions, or by their squared error estimates. If the algorithm converges to the uninteresting all zero solution, it can be restarted with a different choice of starting values.

\subsection{EM for Subset Selection in High-Dimensional Regression}\label{sec:reg}

A similar computational algorithm is also applicable to feature  selection in high-dimensional regression. Consider the following regression model for $y \in \mathbb{R}^n, X\in \mathbb{R}^{n\times p}, \theta \in \mathbb{R}^p$ where $p>n$:
\begin{eqnarray*}
(y \mid X, \theta) \stackrel{ind}\sim \Nor (X\theta, 1), \; (\theta_i \mid u_i, a) \stackrel{ind}\sim \Nor \left(0, \frac{a}{2u_i}\right), \; p(u_i)  =  \frac{1-e^{-u_i}}{2{\pi^{1/2}}u_i^{3/2}},  \; 0< u_i < \infty, \quad a>0.
\end{eqnarray*}
The normal means model $(y_i \mid \theta_i)\stackrel{ind} \sim \Nor (\theta_i, 1)$ of Section~\ref{sec:seq} can be seen to be a special case of the regression model with $p=n$ and $X=I_n$, where $I_n$ is the identity matrix of size $n$. Since there is no change in the hierarchy compared to the normal means model for the latent $u_i$ terms, their conditional expectations remain unchanged. Similarly, the posterior mode  of $a$ has the same form as Section~\ref{sec:seq}. The only change is that the posterior mode for $\theta$ is now given by
$$
\hat \theta \mid a = \left\{X^T X + \mathrm{diag} \left(\frac{2 \tilde u_i}{a}\right)\right\}^{-1} X^T y.
$$
Consequently, the $(t+1)^{\mathrm{th}}$ EM iteration for $t\ge0$ is:
\begin{align*}
\hat a^{(t+1)} \mid \hat\theta_1^{(t)}, \ldots, \hat\theta_p^{(t)} &= \frac{\{\hat a^{(t)}\}^{3/2}}{p\pi} \sum_{i=1}^{p} \left( \frac{1}{\{\hat\theta_i^{(t)}\}^2 +\hat a^{(t)}}\right), \\
\hat \theta^{(t+1)} \mid \hat a^{(t+1)}  &= \left\{ X^{T}X +  \mathrm{diag} \left(\frac{2 \tilde u_i^{(t)}}{\hat a^{(t+1)}}\right)\right\}^{-1} X^T y,
\end{align*}
where, as in Section~\ref{sec:seq},
$$
\tilde u_i^{(t)} =\frac{1}{2\pi{a^{1/2}}} \left(\frac{a}{\theta_i^2} - \frac{a}{\theta_i^2 +a}\right),
$$
computed at $a = \hat a^{(t)}, \theta_i = \hat \theta_i^{(t)}$. The computationally limiting step is the calculation of the inverse of the $p \times p$ matrix of the form $(X^{T}X + D^{-1})^{-1}$ where $D^{-1}$ is a $p \times p$ positive definite diagonal matrix, which in our case is $\mathrm{diag} \left({2 \tilde u_i}/{a}\right)$. The naive computational complexity is $O(p^3)$. However, an application of the Woodbury matrix identity gives
$$
(X^TX+D^{-1})^{-1} = D - DX^T(XDX^T + I_n)^{-1} XD.
$$
This involves the computing the inverse of an $n \times n$ matrix, which is $O(n^3)$, and the computation of matrix products $DX^{T}$ and $X^TD$, which are $O(np^2)$. Thus, the resultant computational complexity is $O(np^2)$ when  $p > n$, which is an improvement over $O(p^3)$.

\subsection{One-step Estimator Using the LLA Algorithm}\label{sec:lla}
We now discuss the implications of the horseshoe-like prior as a Laplace scale mixture (see Lemma~\ref{lemma:laplace}), and show that it is useful for sparse parameter learning via the local linear approximation (LLA) algorithm of \cite{zou2008one} that improves upon the local quadratic approximation (LQA) of \cite{fan2001variable}. In particular, \cite{zou2008one} provided an EM algorithm and an optimal one-step estimator by using an inverse Laplace transform on the bridge penalty, which is equivalent to a Laplace mixture of a stable law. In general, any sparsity-inducing prior that admits a Laplace mixture representation falls into the LLA--LQA framework, a notable example being the generalized double Pareto prior \citep{armagan2011generalized}. 

\cite{hunter2005variable,fan2001variable} and \cite{zou2008one} discuss LQA and LLA algorithms. \cite{hunter2005variable} discuss the relationship of the LQA and minorize-majorize (MM) algorithms which are extensions of the EM algorithm. 
When penalties can be written as a cumulant transformation, equivalently a scale mixture of normals, these algorithms are exact. \cite{polson2015mixtures} discuss the duality between mixture and envelope representation from a Bayesian perspective of hierarchical modeling and present several useful conditions for such duality to hold. 

We discuss these strategies for the horseshoe-like prior after a brief description of the framework in the context of a penalized likelihood model. Specifically, consider the regularization problem
\[
Q( \theta ) = \argmax_{\theta\in\mathbb{R}^p} \left\{\sum_{i=1}^n l_i ( \theta ) - n \sum_{j=1}^p \pi_\tau \left( \abs{\theta_j} \right)\right\}, 
\]
where $l_i(\theta)$ is the log likelihood of the $i$th observation, $n$ is the number of observations, $p$ is the model space dimension and $\pi_\tau$ is the penalty applied to each coefficient, although in principle they could be component-specific. The LQA algorithm uses a quadratic Taylor approximation for $ \pi_\tau \left ( \abs{\theta_j} \right ) $ whereas the LLA algorithm \citep[Equation (2.6)]{zou2008one} uses
\[
 \pi_\tau \left ( | \theta_j | \right ) \approx \pi_\tau \left ( | \theta_j^{(0)} | \right )
 +  \pi_\tau^\prime \left ( | \theta_j^{(0)} | \right ) \left ( | \theta_j | -  | \theta_j^{(0)} | \right ), \; \text{for} \; \theta_j \approx \theta_j^{(0)}. 
\]
Hence, LLA leads to the following iterative algorithm that can be solved with the LARS algorithm \citep{efron2004least} for LASSO. Set the initial value $\theta_j^{(0)}$ to be the un-penalized maximum likelihood estimate. For each $k = 1,2, \ldots$, solve the iterative system of equations until convergence of the $\{ \theta^{(k)} \}$ sequence. 
\beq
\theta^{(k+1)} = \argmax_{\theta\in\mathbb{R}^p} \left\{\sum_{i=1}^n l_i ( \theta ) - n \sum_{j=1}^p \pi_\tau^\prime \left ( \abs{ \theta_j^{(k)} } \right ) \abs{ \theta_j }\right\}. \label{eq:lla}
\eeq
This scheme is called the LLA algorithm of \citet{zou2008one}, which has a unque advantage of producing sparse intermediate and final estimates $\theta^{(k)}$ unlike the LQA algorithm. \citet{zou2008one} also showed that the LLA algorithm can be recast as an EM algorithm under certain conditions. Suppose the exponentiated (negative) penalty function $\exp(-n \pi_\tau(\cdot))$ admits the following Laplace mixture representation: 
\beq
\exp(-n \pi_\tau(\abs{\theta_j})) = \intpos \frac{1}{2\omega_j} {\rm e}^{-\abs{\theta_j}/\omega_j} p(\omega_j) d \omega_j. \label{eq:lapmix}
\eeq
Then, maximizing $Q(\btheta)$ becomes equivalent to calculating the posterior mode of $p(\btheta \mid y)$ by treating $\exp(-n \pi_\tau(\abs{\theta_j}))$ as the prior on $\btheta$ after marginalizing the hyperparameters. This property holds true for the penalty induced by the horseshoe-like prior as it satisfies the Laplace mixture representation (\textit{vide} Lemma \ref{lemma:laplace}). For the general prior-penalty in \eqref{eq:lapmix}, the exact EM step for LLA algorithm is given by \citep[Equation (2.13)]{zou2008one}:
\[
\theta^{(k+1)} = \argmax_{\theta\in\mathbb{R}^p}\left[ \sum_{i=1}^n l_i ( \theta ) +  \sum_{j=1}^p \left\{ - \abs{\theta_j} ~ \E \left( \omega_j^{-1} \mid \theta^{(k)}_j, y \right) \right\}\right], \; k =1,2, \ldots. 
\]
As the posterior moment comes from a scale mixture, the expectation can be derived without an explicit knowledge of the mixing measure. A computationally efficient alternative to the aforesaid EM procedure is the one-step estimator $\hat{\theta}_{ose}$ proposed by \cite{zou2008one}, that automatically incorporates sparsity. For the linear regression model, taking $\theta^{(0)}$ to be the ordinary least squares estimator, we get: 
\beq
\theta_{\text{lin-reg}}^{(1)} = \argmin_{\theta\in\mathbb{R}^p} \left\{\half \vectornorm{y - X\theta}^2 + n \sum_{j=1}^p \pi_\tau^\prime \left ( | \theta_j^{(0)} | \right ) | \theta_j |\right\}, \nonumber
\eeq
and for a general likelihood model, assuming $\theta^{(0)} = \hat{\theta}$(mle), the corresponding one-step estimators are given as: 
\beq
\theta_{\text{log-lik}}^{(1)} = \argmin_{\theta\in\mathbb{R}^p} \left\{ \half (\theta - \theta^{(0)})^{\prime}[-\nabla^2 \ell(\theta^{(0)})] (\theta - \theta^{(0)}) + n \sum_{j=1}^p \pi_\tau^\prime \left ( | \theta_j^{(0)} | \right ) | \theta_j |\right\}, \nonumber
\eeq
For the horseshoe-like prior, $\pi_\tau^\prime \left ( | \theta_j^{(0)} | \right )$ is given as: 
\[
\pi_\tau^\prime \left ( | \theta_j | \right ) = \frac{{4\tau^2/|\theta_j|^3}}{\left(1+{2\tau^2/\theta_j^2}\right)\log \left(1+{2\tau^2}/{\theta_j^2}\right) },
\]
Hence, the one-step estimator for the horseshoe-like prior for the normal means problem can be written as: 
\beq
\theta^{(1)} = \argmin_{\theta\in\mathbb{R}^p} \left\{ \half \vectornorm{y - \theta}^2 + 4\tau^2 n \sum_{j=1}^p  \frac{\abs{\theta_j}}{ \abs{\theta^{(0)}}^3(1+\nicefrac{2\tau^2}{(\theta^{(0)}_j)^2})\log (1+\nicefrac{2\tau^2}{(\theta^{(0)}_j)^2}) } \right\}. \label{eq:ose}
\eeq
The one-step estimator in \eqref{eq:ose} has a superficial similarity with the adaptive LASSO \citep{zou2006adaptive} in that the weights of $\abs{\theta_j}$ are decreasing function of $\hat{\theta}(ols)$. The one-step estimator can be rapidly computed by exploiting the LARS algorithm \citep{efron2004least}.

\section{Computational Algorithm II: MCMC for Posterior Exploration}\label{sec:mc}

In addition to  fast EM and LLA algorithms for MAP estimates, one may wish to explore the entire posterior for a full Bayes solution and uncertainty quantification. The hierarchy for the horseshoe-like prior in \eqref{eq:u_i} can be reparameterized by taking $t_i^2 = 2u_i$ and $\tau^2=a$ to yield the following:
\begin{eqnarray}
(\theta_i \mid t_i, \tau) \sim \Nor \left(0, \frac{\tau^2}{t_i^2}\right), \; p(t_i)= \frac{(1-e^{-\half t_i^2})}{\sqrt{2\pi} t_i^{2}},\label{eq:hslike2}
\end{eqnarray}
where $t_i \in  \mathbb{R},\tau^2 > 0$ and the prior density $p(t_i)$ in \eqref{eq:hslike2} is known as the standard slash-normal ($SN(0,1)$) density, given by: 
\beq
p_{SN}(x) = \frac{\phi(0) - \phi(x)}{x^2} = \frac{1-e^{-\half x^2}}{\sqrt{2\pi}x^2}, \quad x \in \mathbb{R}\nonumber \label{eq:slashpdf},
\eeq
where $\phi(\cdot)$ is the density of a standard normal. 
The $SN(0,1)$ density can be also written as a normal variance mixture with a $\mbox{Pareto}(1/2)$ mixing density \citep{gneiting1997normal,barndorff1982normal}. The following result provides a scale mixture representation for the type II modulated normal density, which reduces to the slash-normal density for $b = 1/2$. 
\begin{proposition} \citep{gneiting1997normal}. \label{res:sn}
Suppose $p(x)$ is a scale mixture of normal with density
\[
p(x) = \int_{0}^{\infty} \frac{1}{(2\pi \nu)^{\half}}\exp \left(-\frac{x^2}{2\nu} \right)d F(\nu),
\]
where $F$ is a distribution function on $[0,\infty]$. The modulated normal distributions of type II arise when $F(\cdot)$ is a Pareto distribution on $[1,\infty)$ with parameter $b>0$. The Pareto distribution has density $b/\nu^{b+1}$ for $\nu > 1$, and the resulting normal scale mixture has density: 
\[
p(x) = \frac{b}{(2\pi)^{1/2}} \left(\frac{x^2}{2} \right)^{-(b+\half)}\gamma\left(b+\half,\frac{x^2}{2}\right).
\]
Here $\gamma(\alpha,x) = \int_{0}^{x} t^{\alpha-1} {\mathrm e}^t dt$ denotes the lower incomplete gamma function. 
\end{proposition}
\noindent Hence, the following lemma is immediate.
\begin{lemma}[Hierarchy for slash-normal]\label{lemma:nvm}
Slash-normal random variables can be generated as the $X = ZV^{\half}$, where $Z$ is a standard normal and $V$ follows a Pareto distribution on $[1,\infty)$ with parameter $1/2$. 
\end{lemma}
\noindent Thus, the final scale mixture representation for the horseshoe-like prior is: 
\begin{align}
(\theta_i \mid t_i, \tau) & \sim \Nor\left( 0, \frac{\tau^2}{ t_i^2} \right), (t_i)  \sim SN(0,1), \; t_i \in \mathbb{R}, \; \tau^2 > 0, \label{eq:hslike3} \\
\text{or, }\nonumber\\
(\theta_i \mid t_i, \tau) & \sim \Nor\left(0, \frac{\tau^2}{ t_i^2}\right), (t_i \mid s_i)  \sim \Nor \left(0, s_i \right),\; s_i \sim \mbox{Pareto}\left({1}/{2} \right), \; t_i \in\mathbb{R}, \; \tau^2 > 0. \label{eq:pareto}
\end{align}

\subsection{Complete Conditionals and an MCMC Sampler}\label{sec:mcmc}
We use the scale-mixture representation of $SN(0,1)$ mixing density from Result~\ref{res:sn}: 
\begin{multline}
\frac{(1- e^{-\half t_i^2})}{\sqrt{2\pi} t_i^2} =  \int_{1}^{\infty} \frac{1}{\sqrt{2\pi s_i}} \exp\left(- \frac{t_i^2}{2 s_i}\right) \frac{1}{2 s_i^{3/2}} d s_i = \int_{0}^{1} \frac{1}{2\sqrt{2\pi}} \exp\left(- \frac{\nu_i t_i^2}{2}\right) d \nu_i, \; \text{where} \; \nu_i = s_i^{-1}.\nonumber
\end{multline}
We need to either specify a prior on the hyper-parameter $\tau$ (full Bayes) or treat it as a tuning parameter (empirical Bayes). Since $\tau$ is a scale parameter for $p(\theta_i)$, one option is a $C^+(0,1)$ prior on $\tau$. We first present the steps in the MCMC scheme conditional on $\tau$, where full conditionals of the other parameters are in closed form and then discuss simulation of $\tau$, which requires a slice sampling step. Together, these steps constitute a Metropolis within Gibbs approach. Conditional on $\tau$, the joint density is:
\begin{eqnarray}
p(y, \theta, t, \nu \mid \tau) &\propto& \prod_{i=1}^{n} \exp\left\{-\frac{(y_i - \theta_i)^2}{2} \right\} \frac{|t_i|}{|\tau|} \exp\left(-\frac{t_i^2}{2\tau^2}\theta_i^2\right) \exp\left(- \frac{\nu_i t_i^2}{2}\right) \mathbf{1}\{0< \nu_i< 1\}. \label{eq:joint} \nonumber
\end{eqnarray}
The complete conditionals given $\tau$ for $i = 1, \ldots, n$ are:
\begin{align*}
(\theta_i \mid y_i, t_i, \nu_i, \tau) & \sim \Nor\left( \left(1+\frac{t_i^2}{\tau^2}\right)^{-1} y_i, \left(1+\frac{t_i^2}{\tau^2}\right)^{-1}  \right), \\
(t_i^2 \mid y_i, \theta_i, \nu_i , \tau) & \sim \mbox{Gamma}\left( \text{shape} = \frac{3}{2}, \mbox{rate} = \frac{\theta_i^2}{2\tau^2}+\frac{\nu_i}{2} \right), \\
(\nu_i \mid y_i, t_i, \theta_i , \tau) & \sim \mbox{Exponential} \left(\mbox{rate} = \frac{t_i^2}{2} \right) \mathbf{1}\{0< \nu_i < 1\}.
\end{align*}
Under the half Cauchy prior for $\tau$, $p(\tau) \propto (1+\tau^2)^{-1}$, the conditional distribution of $\eta = 1/\tau^2$ is given by:
\beq
p(\eta \mid y, \theta, t , \nu) \propto \frac{1}{1+\eta} \eta^{\frac{n-1}{2}} \exp\left( -\frac{\eta}{2} \sum_{i=1}^{n} t_i^2\theta_i^2 \right). \nonumber
\eeq
Thus, the slice sampling steps for sampling $\eta$ are: 
\ben
\item Sample $(u \mid \eta)$ uniformly on $[0,(1+\eta)^{-1}]$.
\item Sample $(\eta \mid u)\sim \mbox{Gamma}((n+1)/2, \sum_{i=1}^{n} t_i^2\theta_i^2/2)$, a Gamma density, truncated to have zero probability outside the interval $[0,(1-u)u^{-1}]$.
\een

\section{Simulation Study}

We performed simulation studies to compare feature selection performances with the normal means model of Section~\ref{sec:seq} and the linear regression model of Section~\ref{sec:reg}. 

\subsection{Normal Means Model} 
We take $n=1000$. In true $\theta$, components 1--10 are of magnitude 3, components 11--20 are of magnitude $-3$, followed by 980 zeros. Then we generate data as $(y_i \mid \theta_i) \stackrel{ind} \sim \Nor(\theta_i, 1)$ for $i=1,\ldots, n$. We compare the horseshoe posterior mode obtained by the proposed EM algorithm of Section~\ref{sec:seq}, posterior mean obtained from the MCMC algorithm of Section~\ref{sec:mcmc}, SCAD, MCP and lasso. The results are summarized in Figure~\ref{fig:seq} and Table~\ref{tab:seq}. The posterior mode correctly identifies 600 out of 880 zero components, which is the highest among all methods. It also identifies 19 of the 20 true non-zero features, indicating a good performance in subset selection. Since a method that performs well in subset selection can have poor $\ell_2$ estimation properties (e.g., hard-thresholding), we also compare the methods for the sum of squared errors (SSE), defined as $\sum_{i=1}^{n} (\hat \theta_i - \theta_i)^2$, where $\hat \theta_i$ is the estimate of $\theta_i$. The mode performs better than the two other non-convex penalties (SCAD and MCP). Lasso performs reasonably well in terms of SSE, but poorly in terms of detection of zeros and non-zeros, resulting in a denser solution. This behavior is well documented for convex penalties. The horseshoe posterior mean does not result in exact zero solutions. However, in terms of the SSE, it has the best performance among all methods. The reason for this can be seen from Figure~\ref{fig:seq}, second panel from left. The bias of the horseshoe posterior mean goes to zero for large signals, whereas for smaller signals, there is stronger shrinkage compared to the lasso, but a smooth shrinkage profile (unlike the mode, SCAD or MCP). Lasso leaves a small but constant bias in the estimates, due its soft thresholding behavior. Finally, in terms of computational time, the proposed EM algorithm is orders of magnitude faster than state of the art non-convex solvers such as the R package \texttt{ncvreg}, which implements both SCAD and MCP or \texttt{sparsenet}, which implements coordinate descent algorithm to fit MCP.

\begin{figure*}[!t]
   	\centering
	\includegraphics[width=\textwidth, height=6cm]{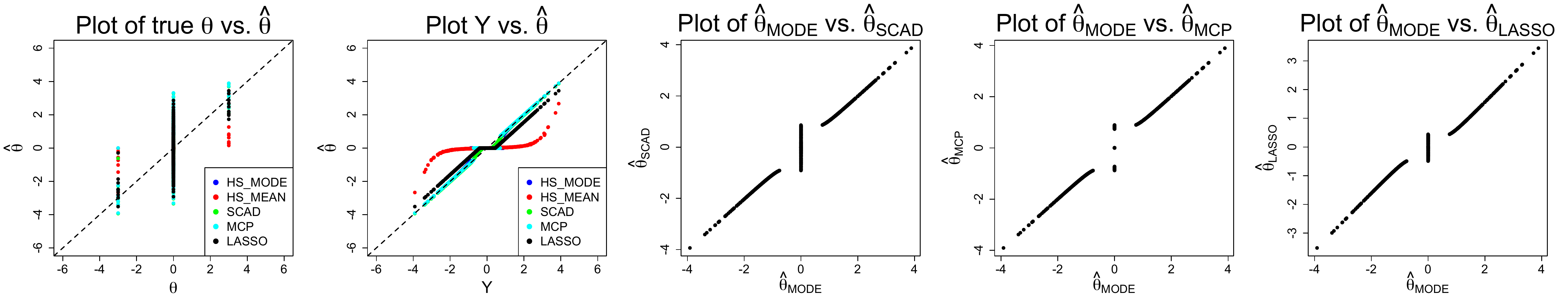}
   	\caption{Simulation results for competing methods on sparse normal means model.}
   	\label{fig:seq}
   \end{figure*}
     \begin{table}[!h]
\begin{center}
\begin{sc}
\begin{tabular}{rccccl}
\hline
 	 	 	 	 	 	 	 & mode  & mean  & Scad  & Mcp  & Lasso \\ 
\hline
 SSE 	 	 	& 890.67 & 111.8 & 1006.45 & 977.98 & 540.5 \\ 
 Cor\_Z 	 	 & 600 & NA & 163 & 515 & 310 \\  
Cor\_Nz 	  & 19 & NA & 20 & 19 & 20 \\  
Time 	 	  	 & 0.82 & 10.75 & 13.21 & 11.637 & 3.16 \\  
\hline
\end{tabular}
\end{sc}
\end{center}
\vskip -0.1in
\caption{Performance comparisons in normal means model for HS posterior mode, HS posterior mean, SCAD, MCP and LASSO. The rows are sum of squared error (SSE), zeros and non-zeros correctly detected (COR\_Z \& COR\_NZ) and time in s. (TIME). \label{tab:seq}}
\end{table}

\subsection{Linear Regression Model} 
We take $n=70, p=350$. The true $\theta\in \mathbb{R}^p$ has components 1--10 are of magnitude 3, components 11--20 are of magnitude $-3$, followed by 330 zeros. The matrix of predictors $X\in \mathbb{R}^{n\times p}$ is generated from i.i.d. standard normals. Finally, the observations are generated as $Y_i \stackrel{ind}\sim \Nor(X\theta, 1)$ for $i=1,\ldots, n$. Figure~\ref{fig:reg} and Table~\ref{tab:reg} document the results. Here the posterior mode has the second best performance in detection of zeros. SCAD detects the highest number of true zeros correctly, but this comes at the expense of a poor performance in detection of non-zeros (7 out of 20) and a poor SSE. The horseshoe posterior mode results in sparser solution compared to MCP and lasso. The mode, MCP and lasso all perform well in the detection of non-zeros. Computation times for all methods (except MCMC) are comparable. As before, the mean has the lowest SSE, but does not give a sparse solution. The poor fit of SCAD in this case can be verified from the second panel from left of Figure~\ref{fig:reg}, where the fitted $\hat Y = X\hat \theta$ values can be seen to be far away from the actual $Y$ values for SCAD.

\begin{figure*}[t!]
   	\centering
	\includegraphics[width=\textwidth, height=6cm]{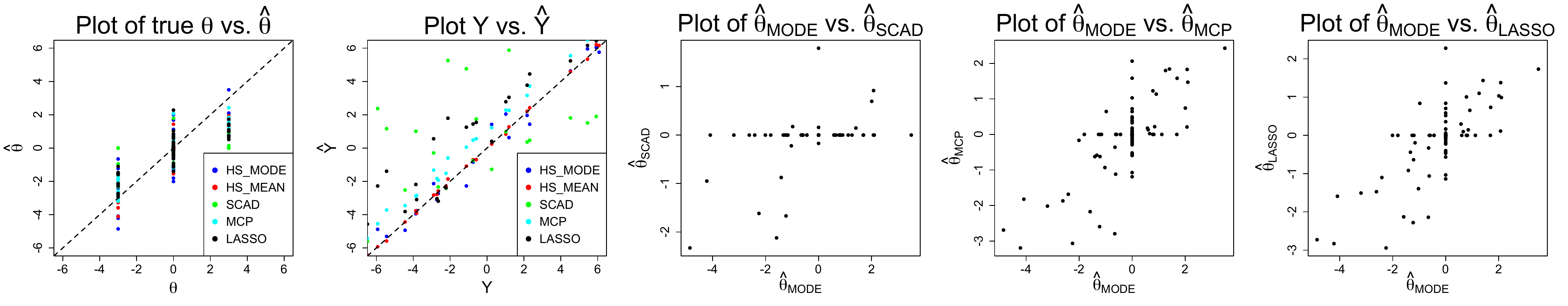}
   	\caption{Simulation results for competing methods on sparse linear regression model.}
   	\label{fig:reg}
   \end{figure*}

\begin{table}[!h]
\begin{center}
\begin{sc}
\begin{tabular}{rccccl}
\hline
 	 	 	 	 	 	 	 & mode  & mean  & Scad  & Mcp  & Lasso \\ 
\hline
 SSE 	 	 	& 91.03 & 44.35 & 143.2 & 42.55 & 66.93 \\ 
 Cor\_Z 	 	 & 302 & NA & 323 & 276 & 292 \\  
Cor\_Nz 	  & 18 & NA & 7 & 20 & 20 \\  
Time 	 	  	 & 0.248 & 14.978 & 0.226 & 0.561 & 0.177 \\  
\hline
\end{tabular}
\end{sc}
\end{center}
\vskip -0.1in
\caption{Performance comparisons in the regression model for HS posterior mode, HS posterior mean, SCAD, MCP and LASSO. The rows are sum of squared error (SSE), zeros and non-zeros correctly detected (COR\_Z \& COR\_NZ) and time in s. (TIME). \label{tab:reg}}
\end{table}

\subsection{Comparisons With The Horseshoe Prior}
Since the horseshoe-like prior is a close approximation to the horseshoe prior, it is perhaps instructive to take a closer look at a comparison between the two.
We first demonstrate the performance of the horseshoe-like prior in  a simulation study for estimating a sparse normal mean vector with $(y_i \mid \theta_i) \sim \Nor(\theta_i, \sigma^2)$ and two different choices of $\btheta$: (1) $\btheta_1 = (\underbrace{7,\ldots,7}_{q_n=10},\overbrace{0,\ldots,0}^{n-q_n = 90})$ and (2) $\btheta_2 = (\underbrace{7,\ldots,7}_{q_{n}=10},\underbrace{3,\ldots,3}_{r_n=10}\overbrace{0,\ldots,0}^{n-q_n-r_n = 80})$. The choices are made to test the performance of horseshoe-like prior with sparse signals near the `verge of detectability' $\sqrt{2 \log n}$ \citep{bogdan2011asymptotic} as well as for signals with a relatively large magnitude, e.g. $2\sqrt{2\log n}$ away from origin. Similar to the horseshoe prior, the horseshoe-like prior should be able to identify the signals in both cases. Figure~\ref{fig:sim1} shows the estimated $\hat{\theta}$ under the  horseshoe-like prior, along with the observations $y_i$s and the $95\%$ credible intervals. It is evident that the true signals are recovered in both the cases.

It is also instructive to compare the shrinkage profile of the horseshoe-like prior with that of the horseshoe prior for the second example. Figure~\ref{fig:sim2} shows that although the shrinkage profile for the two priors are very similar, the horseshoe-like prior exerts a slightly stronger shrinkage on the noise observations near zero, but does not shrink the signals both near and far from the $\sqrt{2 \log n}$ boundary.

\begin{figure}[!t]%
\centering
\includegraphics[height=3in,width=\columnwidth]{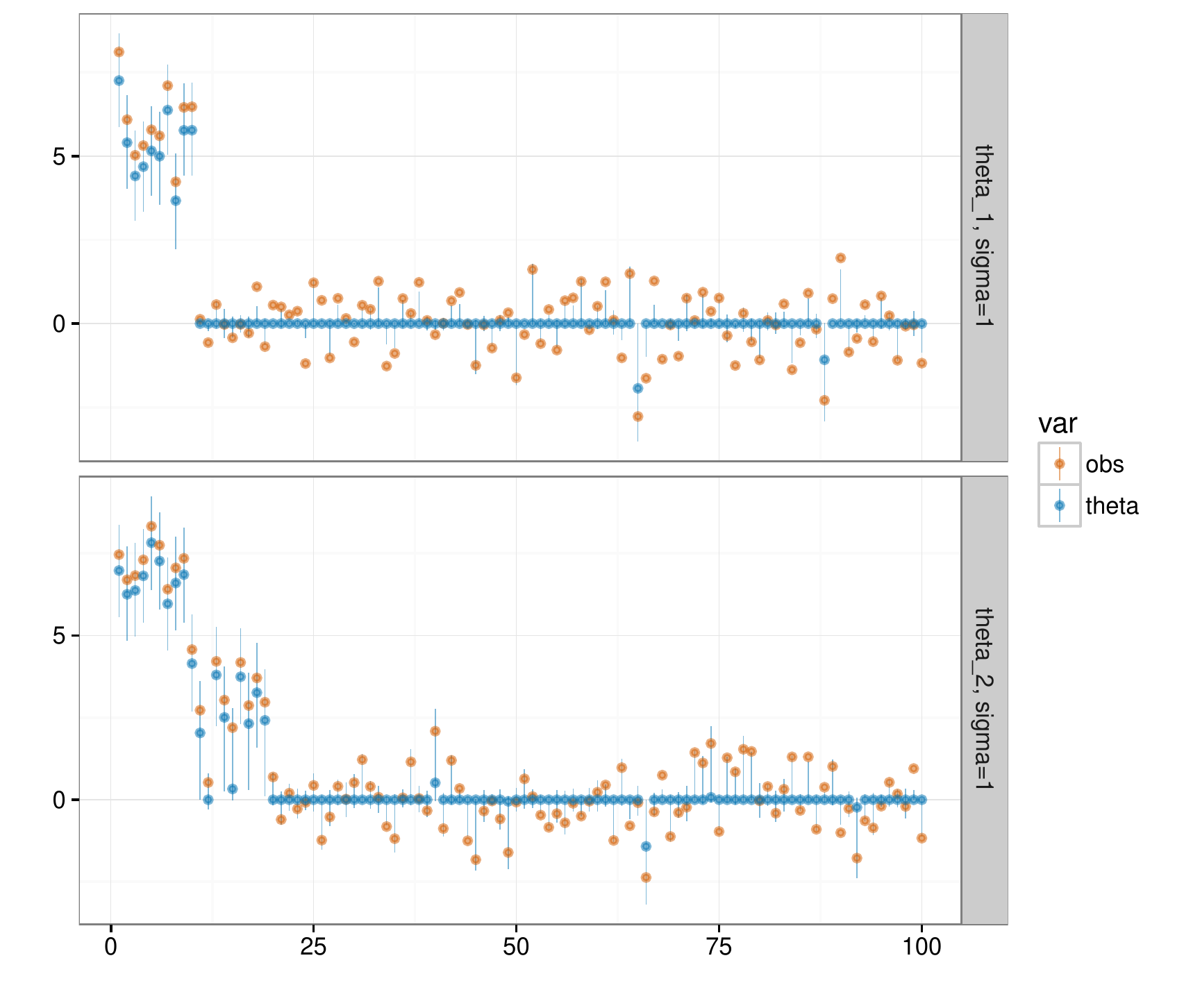}%
\caption{Comparison of posterior mean estimates for two different sparse normal means, $\btheta_1 \sim 0.1 \delta_{\{7\}}+0.9\delta_{\{0\}}$ (top) and $\btheta_2 \sim 0.1 \delta_{\{7\}}+0.1\delta_{\{3\}}+0.8\delta_{\{0\}}$ (bottom) under the horseshoe-like prior. }%
\label{fig:sim1}%
\end{figure}

\begin{figure}[!h]%
\centering
\includegraphics[height=2.3in,width=0.7\textwidth]{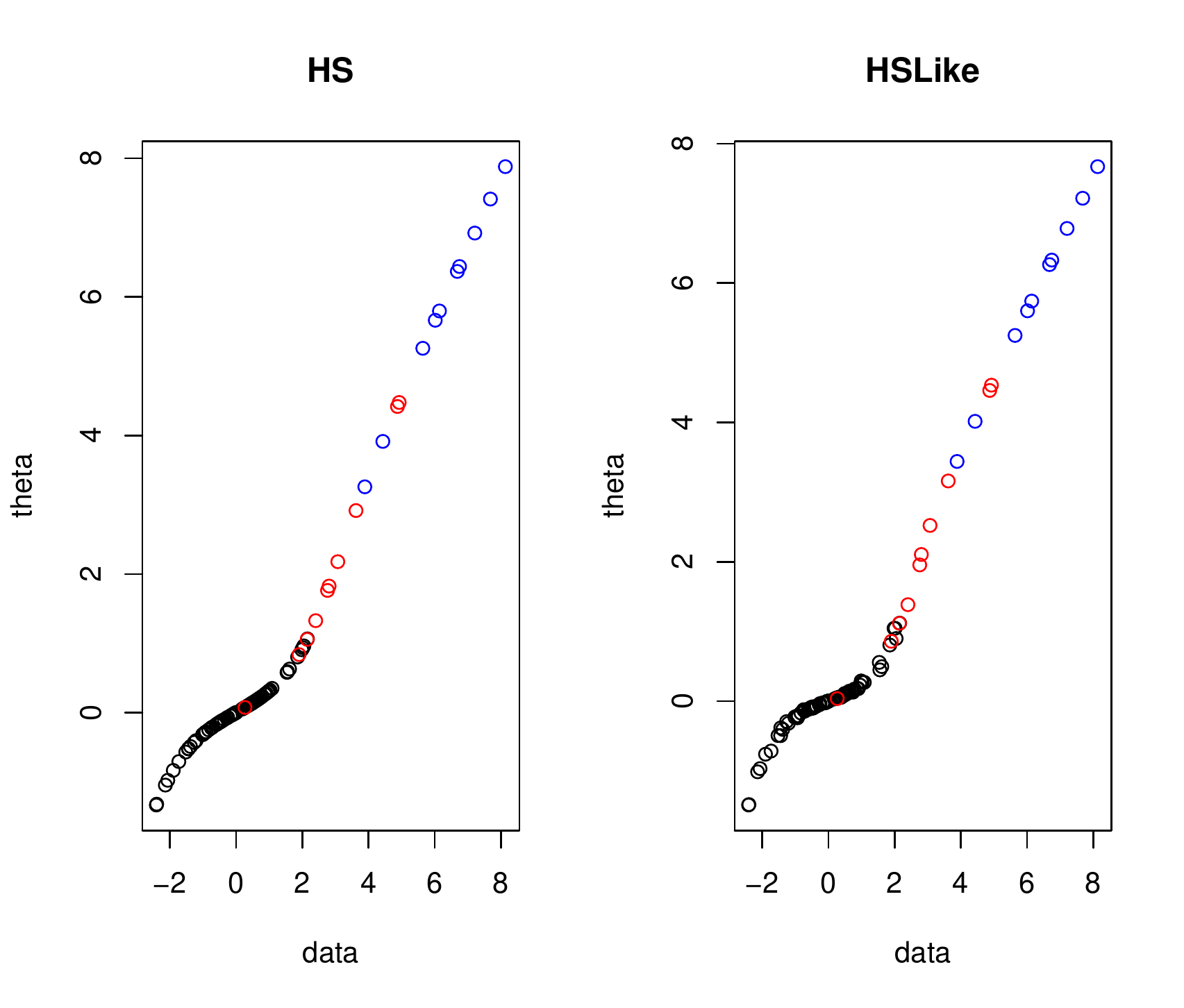}%
\caption{Comparison of posterior mean estimates under the horseshoe (HS) and horseshoe-like (HSLike) priors for $\btheta \sim 0.1 \delta_{\{7\}}+0.1\delta_{\{3\}}+0.8\delta_{\{0\}}$. The black, red and blue circles represent true $\theta_i = 0, 3$ and $7$ respectively. }%
\label{fig:sim2}%
\end{figure}

\section{Leukemia Data Example}
We consider a popular microarray gene expression data set with 3051 genes and 38 leukemia samples \citep{golub1999molecular, dudoit2002comparison}. This is a two class study where the goal is to identify genes that significantly differ between the 27 acute lymphoblastic leukemia (ALL) cases and 11 acute myeloid leukemia (AML) cases. The multiple testing for this data is carried out as follows: first a two-sample $t$-test with 36 degrees of freedom was performed for each 3,051 genes and the $t$-test statistics are converted to $z$-test statistics using the quantile transformation $z_i = \Phi^{-1}(T_{36~\mbox{d.f.}}(t_i))$ for $i = 1, \ldots,  3051$. The $i^{\text{th}}$ null hypothesis $H_{0i}$ posits no difference in the gene expression levels for the $i^{\text{th}}$ gene between the ALL and AML cases, and under the global null hypothesis $\cap H_{0i}$ the histogram of the $z$-values should mimic a $\Nor(0,1)$ curve closely. The histogram of the z-values along with the standard normal curve and a fitted normal density are shown in Figure~\ref{fig:hist}. The departure of the histogram from the normal density curve suggests presence of many genes \emph{differing} between the two classes. 

\begin{figure}[t!]%
\centering
\includegraphics[width=12cm, height=5cm]{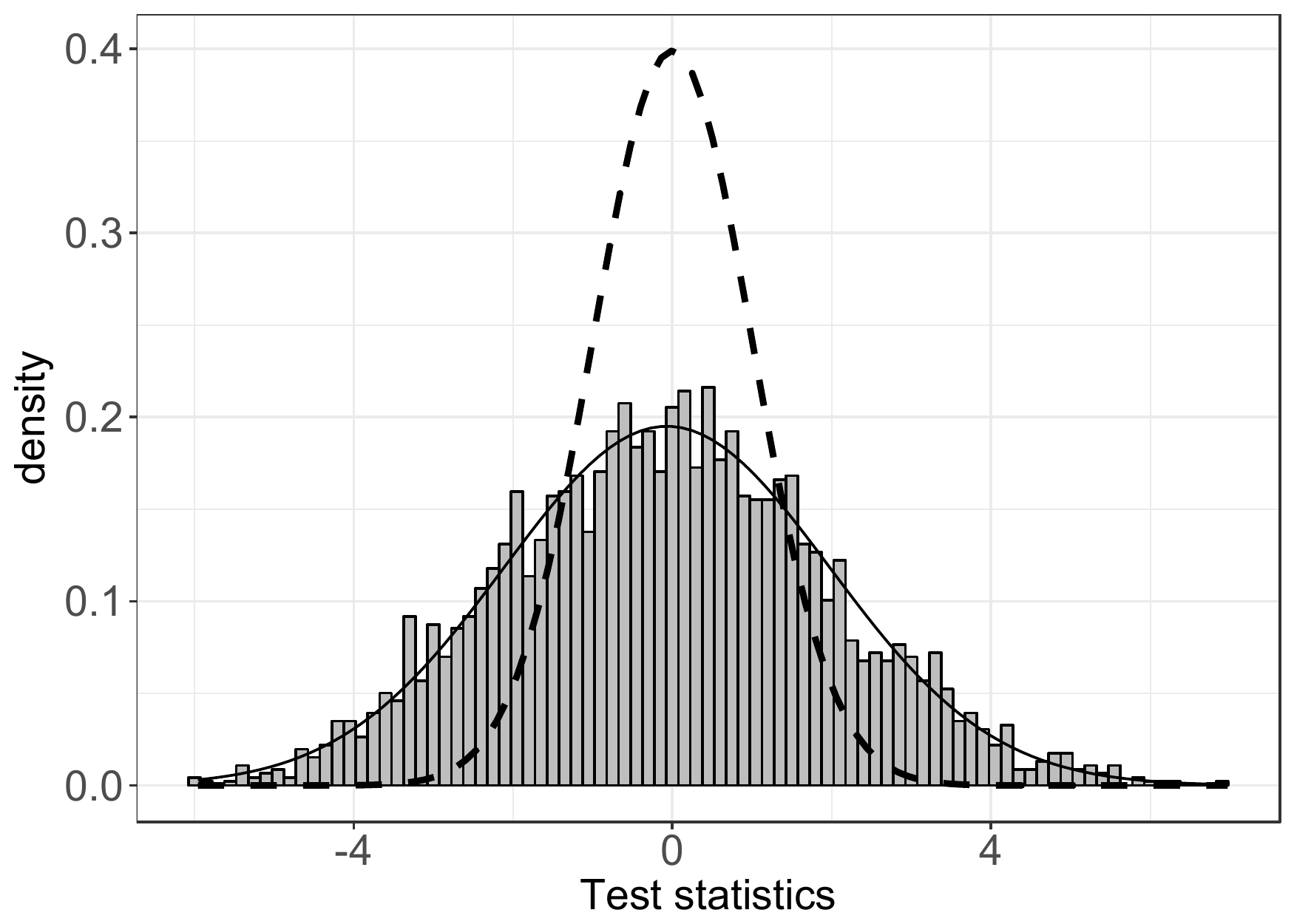}%
\caption{Histogram of $z$-values along with a dashed $\Nor(0,1)$ and a solid $\Nor(\bar z, s)$ density curve, where $\bar z$ and $s$ are the sample mean and standard deviation of the $z$-values.}%
\label{fig:hist}%
\end{figure}

The three classical multiple testing procedures, \textit{viz.} Bonferroni, Benjamini--Hochberg and Bejamini--Yekutieli, identify 98, 681 and 269 genes as significant, by adjusting $p$-values obtained from the test statistics. Given the size of the data, the Bonferroni procedure is overly conservative for large scale testing and Benjamini--Hochberg can be thought of as a recognized gold standard \citep[see, e.g.,][]{efron2010large}. In order to perform subset selection for this data, we compare the horseshoe posterior mode obtained by the proposed EM algorithm along with SCAD, MCP and lasso. It is also possible to use the posterior mean of the horseshoe-like prior with a thresholding rule as in \citet{datta2013asymptotic} for performing multiple testing, but we do not consider it here since it is not a formal subset selection algorithm. Figure~\ref{fig:thetahat} compares the thresholding nature for the candidate methods and shows that the lasso is least conservative (declares 1,395 genes significant)  and the horseshoe posterior mode is the most conservative (declares 987 genes significant) among them. Also, it appears that the three methods except the lasso induce somewhat similar thresholding rules. Figure~\ref{fig:facet} plots the estimated mean parameter $\hat{\theta_i}$s underlying the normal observations $z_i$s with the points color-coded according to the Benjamini--Hochberg multiple testing rule. Once again, it seems that the horseshoe posterior mode performs similarly to SCAD and MCP and the lasso acts in an anti-conservative way, potentially leading to many false discoveries. 

\begin{figure}[t!]%
\centering
\includegraphics[width=12cm, height=5cm]{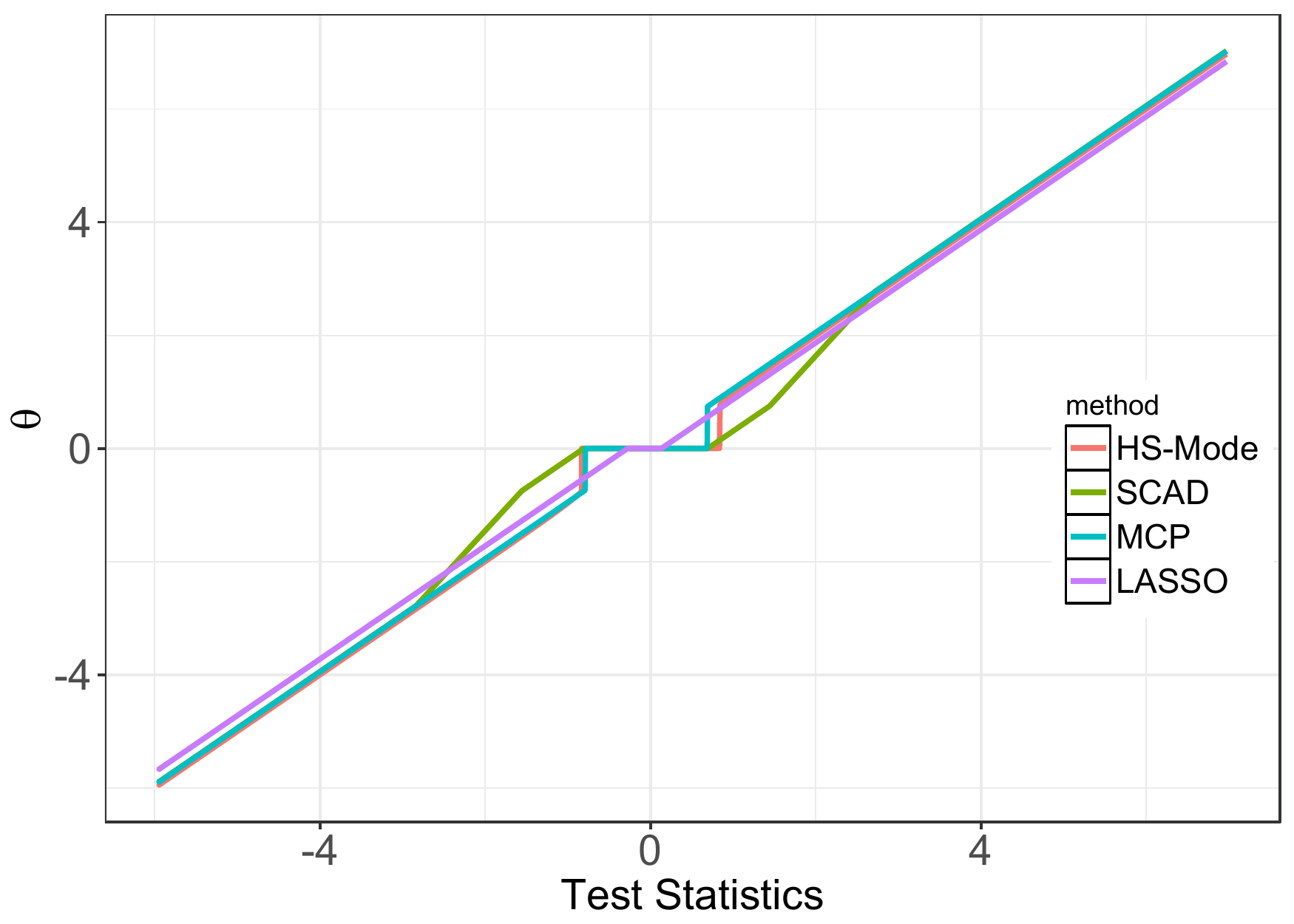}%
\caption{The posterior estimates for the competing methods versus the observed test statistics.}%
\label{fig:thetahat}%
\end{figure}

\begin{figure}[t!]%
\centering
\includegraphics[width=11cm, height=5cm]{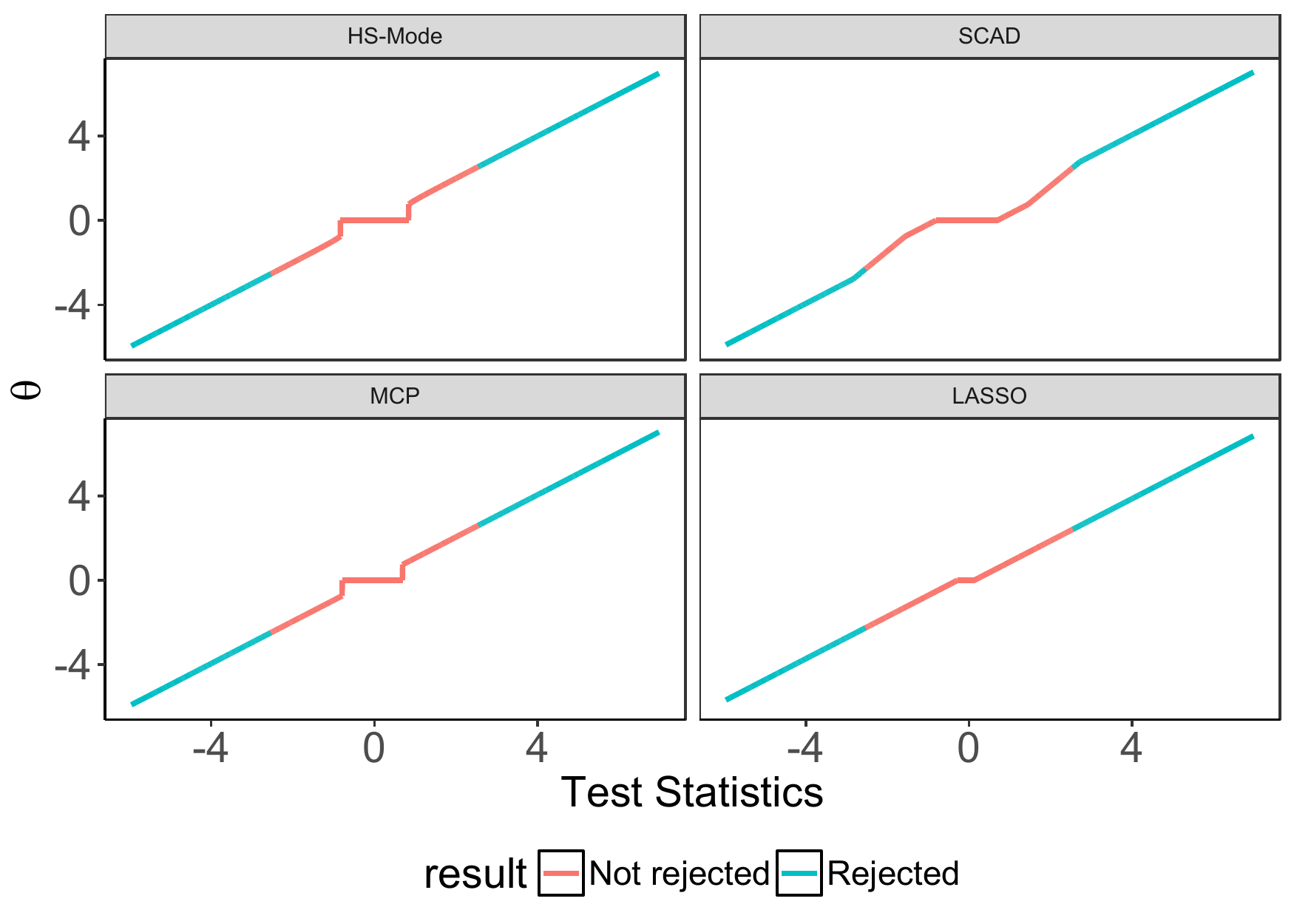}%
\caption{The posterior estimates versus the observed test statistics for different methods, with the points color-coded according to the Benjamini-Hochberg decision rule.}%
\label{fig:facet}%
\end{figure}

\section{Conclusions and Future Work}
We developed novel theoretical insights and fast computational algorithms for subset selection using  the horseshoe regularization penalty. Our approach has a probabilistic representation, which allows for simulating the entire posterior via MCMC, (Section~\ref{sec:mc}), in addition to developing EM and LLA algorithms for identifying MAP point estimates (Section~\ref{sec:em}). In turn, this allows us to contrast the respective strengths and weaknesses of posterior mean and posterior mode. The former typically performs best in estimation under squared error loss, but is not sparse. These attributes are exactly reversed for the latter. In terms of both computational speed and statistical performance, horseshoe regularization outperforms state of the art non-convex solvers such as MCP or SCAD. 

There are a number of directions for future work. For example, some other global-local priors that have shown promise in sparse Bayesian inference include the generalized beta \citep{armagan2011generalized}, the horseshoe+ \citep{bhadra2015horseshoe+, bhadra2015default} and the Dirichlet--Laplace \citep{bhattacharya2014dirichlet}, to name a few. An open question is how these priors perform in terms of subset selection and whether fast computational algorithms are available. Following the recommendation of \citet{gelman2006prior}, we used a standard half-Cauchy ($C^{+} (0,1)$) prior in (\ref{eq:horsedens}), similar to the original horseshoe formulation \citep{carvalho2009handling, polson2012half}. However, results in \citet{piironen2016hyperprior} indicate it will be interesting to investigate the effect of the hyper-parameter $\eta$ in a $C^{+}(0,\eta)$ prior in subset selection.

A more general family of proper prior densities  can be constructed as follows:
\begin{eqnarray*}
p(\theta_i\mid \tau)  &\propto& 
 \begin{cases}
     \frac{1}{\theta_i^{1-\epsilon}} \log\left(1 + \frac{\tau^2}{\theta_i^2}\right), & \text{if  } |\theta_i| < 1,\\
     \theta_i^{1-\epsilon} \log\left(1 + \frac{\tau^2}{\theta_i^2}\right),  & \text{if  } |\theta_i| \ge 1,
    \end{cases}
    \end{eqnarray*}
for $\epsilon\ge 0, \tau>0$, which reduces to the horseshoe-like prior of Equation (\ref{eq:ptilde}) for $\epsilon=1$. Furthermore, the density is approximately equal to $\theta_i^{1-\epsilon}\log(\theta_i^{-1})$ near the origin and the tails decay as $\theta_i^{-(1+\epsilon)}$. Thus, the main features of the horseshoe prior, that is, unboundedness at the origin and polynomially decaying tails, are preserved. The parameter $\epsilon$ represents a tradeoff between tail-heaviness and peakedness at the origin. For $\epsilon\in (0,1)$, the tails are heavier compared to the horseshoe, but at the cost of a 
smaller peak at the origin. The opposite is true for $\epsilon>1$. Detailed investigation of this broader class of priors should be considered future work.

\appendix
\setcounter{equation}{0}
\setcounter{table}{0}
\setcounter{section}{0}
\setcounter{figure}{0}
\renewcommand{\theequation}{A.\arabic{equation}}
\renewcommand{\theresult}{A.\arabic{result}}
\renewcommand{\thesubsection}{A.\arabic{subsection}}
\renewcommand{\thelemma}{A.\arabic{lemma}}

\section{Proof of Proposition~\ref{prop:marginal}}\label{app:marginal}
The hierarchy for the horseshoe-like prior can be written as:
\begin{align*}
y_i \mid \theta_i, \sigma^2 & \sim \NormRV(\theta_i, \sigma^2), \; \text{where} \; \sigma^2 \sim \text{Inverse-Gamma}(\alpha,\beta), \\
p(\theta_i \mid \tau) & =  \frac{1}{2 \pi \tau}\log \left( 1 + \frac{\tau^2}{\theta_i^2} \right).
\end{align*}
Here we treat $\tau^2$ as a tuning parameter. The marginal density of $y_i$ is: 
\beq
m(y_i \mid \tau) = \intreal \intpos \frac{1}{\sqrt{2 \pi} \sigma} e^{-\half \frac{(y_i - \theta_i)^2}{2\sigma^2}} \frac{1}{2 \pi \tau}\log \left( 1 + \frac{\tau^2}{\theta_i^2} \right) \frac{\beta^{\alpha}}{\Gamma(\alpha)} (\sigma^2)^{-\alpha-1} e^{-\frac{\beta}{\sigma^2}} d\theta_i d \sigma^2. \label{eq:full}
\eeq
First, integrating out $\sigma^2$ under the $\text{Inverse-Gamma}(\alpha,\beta)$ hyper-prior gives the marginal likelihood: 
\beq
p(y_i \mid \theta_i, \alpha, \beta) = \frac{\Gamma(\alpha+\half)}{\Gamma(\alpha)\Gamma(\half)} \beta^{-1/2} \frac{1}{ \left(1+\frac{1}{2\beta}\left(y_i-\theta_i \right)^2\right)^{\alpha+\half}}, \; \alpha, \beta > 0. \label{eq:int-like}
\eeq
For the special case $\alpha = \beta = \half$, the marginal $t$ distribution of $y_i$ in \eqref{eq:int-like} is into a Cauchy distribution with location $\theta_i$, i.e. 
\[
p(y_i \mid \theta_i) = \frac{1}{\pi\{ 1+(y_i - \theta_i)^2 \}}.
\]
Now, using Lemma \ref{lemma:cauchy}, we can write the horseshoe-like prior as a $\UnifRV(0,1)$ scale mixture of Cauchy, or, in other words, $(\theta_i \mid \lambda_i, \tau) \sim \CauchyRV(0,\lambda_i \tau)$ and $\lambda_i \sim \UnifRV(0,1)$. This hierarchy implies that the horseshoe-like prior is a member of the global-local mixtures described in \cite{bhadra2016global}, where the local shrinkage parameter has a $\UnifRV(0,1)$ prior, commonly used for the global shrinkage parameter $\tau$. In a recent article, \cite{van2016many} argue that restriction of the prior mass of $\tau$ to the interval $[1/n,1]$ helps in achieving near-minimax rates as well as preventing degeneracy of the estimates of $\tau$. Using this scale mixture representation in the hierarchy we can write: 
\[
m(y_i \mid \tau) = \intreal \intunit \frac{1}{\pi\{ 1+(y_i - \theta_i)^2 \}} \frac{1}{\pi} \frac{\lambda_i \tau}{\lambda_i^2 \tau^2 + \theta_i^2} d\lambda_i. 
\]
Equivalently, 
\begin{align*}
Y_i - \theta_i & \sim \CauchyRV(0,1), \; \theta_i \sim \CauchyRV(0,\lambda_i \tau), \; \text{and } \lambda_i \sim \UnifRV(0,1), \\
\Rightarrow Y_i = (Y_i - \theta_i) & + \theta_i \iidd \CauchyRV(0,1) + \lambda_i \tau \; \CauchyRV(0,1) \iidd \CauchyRV(0, 1+\lambda_i\tau) , \; \text{and } \lambda_i \sim \UnifRV(0,1).
\end{align*}
The last equation follows from the following lemma (\textit{vide} \cite{bhadra2016global} for a proof using the \CS{} integral identity). 
\begin{lemma}
  Let $X_i \sim \CauchyRV(0,1)$ $(i = 1, 2)$ be Cauchy distributed random variates, then $Z = w_1 X_1 + w_2 X_2 \sim \CauchyRV( 0, w_1 + w_2).$ where $w_1,w_2 > 0$.
\end{lemma}
Hence the marginal of $y_i$ is:
\begin{align*}
m(y_i \mid \tau) & = \intunit \frac{1}{\pi(1+\lambda_i \tau)\left[ 1+ \left\{ \frac{y_i}{(1+\lambda_i \tau)} \right\}^2 \right]} d\lambda_i \\
& = \frac{1}{\pi} \intunit \frac{(1+\lambda_i \tau)}{\left\{ (1+\lambda_i \tau)^2+ y_i^2 \right\} } d\lambda_i \\
& = \frac{1}{2\pi\tau} \int_1^{(1+\tau)^2} \frac{dt}{t + y_i^2} = \frac{1}{2\pi\tau} \log \left(1 + \frac{\tau^2}{1+y_i^2} \right). 
\end{align*}

\bibliographystyle{biom}
\bibliography{horseshoe-review,horseshoe-plus,hslike_ref}

\end{document}

%% file: math-commands.tex
\usepackage{suffix}

\newcommand{\E}{{\bf E}}

\newcommand{\Nor}{{\cal N}}  

\newcommand{\NormRV}{\mathcal{N}}

\newcommand{\CauchyRV}{\mathcal{C}}

\newcommand{\UnifRV}{\mathcal{U}}

\newcommand{\btheta}{\boldsymbol{\theta}}

\def\CS{Cauchy-Schl\"omilch}

\newcommand{\intreal}{\int_{-\infty}^{\infty}}
\newcommand{\intpos}{\int_{0}^{\infty}}
\newcommand{\intunit}{\int_{0}^{1}}

\newcommand{\ben}{\begin{enumerate}}
\newcommand{\een}{\end{enumerate}}
\newcommand{\beq}{\begin{equation}}
\newcommand{\eeq}{\end{equation}}

\newcommand{\half}{\frac{1}{2}}

\newcommand{\argmin}{\operatornamewithlimits{argmin}}
\newcommand{\argmax}{\operatornamewithlimits{argmax}}

\newcommand{\abs}[1]{\lvert#1\rvert}

\newcommand{\vectornorm}[1]{\left|\left|#1\right|\right|}

\newcommand{\iidd}{\stackrel{\mathrm{D}}{=}}

\DeclarePairedDelimiterX\MeijerM[3]{\lparen}{\rparen}%
{\begin{smallmatrix}#1 \\ #2\end{smallmatrix}\delimsize\vert\,#3}

\newcommand\MeijerG[8][]{%
  G^{\,#2,#3}_{#4,#5}\MeijerM[#1]{#6}{#7}{#8}}

\WithSuffix\newcommand\MeijerG*[7]{G^{\,#1,#2}_{#3,#4}\MeijerM*{#5}{#6}{#7}}

\DeclarePairedDelimiterX\pFqM[3]{\lparen}{\rparen}%
{\begin{smallmatrix}#1 \\ #2\end{smallmatrix}\delimsize\vert\,#3}

\newcommand\pFq[6][]{%
  {}_{#2}F_{#3}\pFqM[#1]{#4}{#5}{#6}}

\WithSuffix\newcommand\pFq*[5]{{}_{#1}F_{#2}\pFqM*{#3}{#4}{#5}}

\numberwithin{theorem}{section}

\numberwithin{Def}{section}

\numberwithin{remark}{section}

\newtheorem{proposition}{PROPOSITION}
\numberwithin{proposition}{section}
\newtheorem{lemma}{LEMMA}
\numberwithin{lemma}{section}

\numberwithin{Cor}{section}

\usepackage{booktabs,array}

\newcount\rowc

\makeatletter
\def\ttabular{%
\hbox\bgroup
\let\\\cr
\def\rulea{\ifnum\rowc=\@ne \hrule height 1.0pt \fi}
\def\ruleb{
\ifnum\rowc=1\hrule height 1.0pt  \else
\ifnum\rowc= 3  \hrule height 0.5pt \else
\ifnum\rowc= 5  \hrule height 0.5pt \else
\ifnum\rowc= 7  \hrule height 0.5pt \else
\ifnum\rowc= 9  \hrule height 0.5pt \else
\ifnum\rowc= 11  \hrule height 0.5pt 
  \else \hrule height 0pt
\fi\fi\fi\fi\fi\fi}
\valign\bgroup
\global\rowc\@ne
\rulea
\hbox to 7em{\strut \hfill##\hfill}%
\ruleb
&&%
\global\advance\rowc\@ne
\hbox to 7em{\strut\hfill##\hfill}%
\ruleb
\cr}
\def\endttabular{%
\crcr\egroup\egroup}